\def\year{2019}\relax
\documentclass[letterpaper]{article} 
\usepackage{aaai19}  
\usepackage{times}  
\usepackage{helvet}  
\usepackage{courier}  
\usepackage{url}  
\usepackage{graphicx}  
\usepackage[utf8]{inputenc} 
\usepackage[T1]{fontenc}    
\usepackage{hyperref}       
\usepackage{url}            
\usepackage{booktabs}       
\usepackage{amsfonts}       
\usepackage{nicefrac}       
\usepackage{microtype}      

\usepackage{amsmath}
\usepackage{amsthm}
\usepackage{amssymb}
\usepackage{framed}
\usepackage{algorithm}
\usepackage{algpseudocode}
\usepackage{breqn}
\usepackage{forloop}
\usepackage[capitalize]{cleveref}
\usepackage{graphicx}
\usepackage{subcaption}
\usepackage{xspace}
\usepackage{bm}
\usepackage{mathtools}
\usepackage[title]{appendix}
\usepackage{natbib}
\usepackage{threeparttable}
\DeclareMathOperator{\rank}{rank}

\crefname{appsec}{Appendix}{Appendices}

\usepackage{xargs}                      

\usepackage[pdftex,dvipsnames,table]{xcolor}  
\usepackage[colorinlistoftodos,prependcaption,textsize=normal]{todonotes}
\newcommandx{\unsure}[2][1=]{\todo[linecolor=red,backgroundcolor=red!25,bordercolor=red,#1]{#2}}
\newcommandx{\change}[2][1=]{\todo[linecolor=blue,backgroundcolor=blue!25,bordercolor=blue,#1]{#2}}
\newcommandx{\info}[2][1=]{\todo[linecolor=OliveGreen,backgroundcolor=OliveGreen!25,bordercolor=OliveGreen,#1]{#2}}
\newcommandx{\improvement}[2][1=]{\todo[linecolor=Plum,backgroundcolor=Plum!25,bordercolor=Plum,#1]{#2}}
\newcommandx{\thiswillnotshow}[2][1=]{\todo[disable,#1]{#2}}

\newcommand{\constraint}{\mathcal{K}}

\DeclareMathOperator*{\argmin}{arg\,min}
\DeclareMathOperator{\expect}{\mathbb{E}}

\newcommand{\AlgFKM}{\textsf{FKM}\xspace}
\newcommand{\AlgUn}{\textsf{Unregularized}\xspace}
\newcommand{\AlgOCG}{\textsf{StochOCG}\xspace}

\newcommand{\defcal}[1]{\expandafter\newcommand\csname 
C#1\endcsname{{\mathcal{#1}}}}
\newcommand{\defbb}[1]{\expandafter\newcommand\csname 
B#1\endcsname{{\mathbf{#1}}}}
\newcounter{calBbCounter}
\forLoop{1}{26}{calBbCounter}{
	\edef\letter{\Alph{calBbCounter}}
	\edef\sletter{\alph{calBbCounter}}
	\expandafter\defcal\letter
	\expandafter\defbb\letter
	\expandafter\defcal\sletter
	\expandafter\defbb\sletter
}

\newtheorem{theorem}{Theorem}
\newtheorem{lemma}{Lemma}

\algnewcommand{\commentcolor}[1]{\Comment{\textcolor{NavyBlue}{#1}}}

\nocopyright
\begin{document}
\title{Projection-Free Bandit Convex Optimization}
\author{
	Lin Chen$ ^{1,2} $\thanks{Equal contribution}\qquad Mingrui Zhang$ ^{3} 
	$\footnotemark[1]\qquad 
	Amin  Karbasi$ ^{1,2} $ \\
	$ ^1 $Yale Institute for Network Science, $ ^2 $Department of Electrical 
	Engineering,\\
	$ ^3 $Department of Statistics and Data Science,
	Yale University\\
	\texttt{\{lin.chen, mingrui.zhang, amin.karbasi\}@yale.edu} \\
}
\maketitle

\begin{abstract}
    In this paper, we propose the first computationally efficient 
    projection-free algorithm for  
    bandit convex optimization (BCO). We show that our algorithm achieves a 
    sublinear regret  of $ O(nT^{4/5}) $ (where $ T $ is the horizon and $ n $ 
    is the dimension) for any 
    bounded  convex functions with uniformly bounded gradients. We also evaluate the 
    performance of our algorithm against baselines on both synthetic and real 
    data sets for quadratic programming, portfolio 
    selection and matrix completion problems. 
\end{abstract}
\section{Introduction}
The online  learning setting models a dynamic optimization process in which data becomes 
available in a sequential manner and the learning algorithm has to adjust and 
update its 
predictor as more data is disclosed. 
It can be best formulated as a repeated two-player  game between a learner and an adversary as 
follows.  At 
each 
iteration $t$, the learner commits to a decision $ \Bx_t $ from a constraint set $ \constraint 
\subseteq
\mathbb{R}^n$. Then, the adversary  selects a cost function $f_t$ and the learner suffers the loss $ 
f_t(\Bx_t) $ in addition to receiving feedback. In the online learning model, it is generally assumed that the learner 
has access to a gradient oracle for all loss 
functions $f_t$, and thus knows  the loss had she chosen a different point  at 
iteration $t$. The performance of an 
online learning algorithm is measured by a game theoretic metric known as 
regret which is defined as the gap between the total 
loss that the learner has incurred after $T$ iterations and that of the best 
fixed decision in hindsight.

 In online learning, we are usually interested in sublinear regret as a function of the horizon $T$. 
 To this end, other 
structural assumptions are made. 
For instance, when all the loss functions $f_t$, as well as  the constraint set 
$\constraint$, are convex, the problem is 
known as 
 Online Convex Optimization (OCO)~\citep{zinkevich2003online}. This framework 
 has received a lot of 
 attention due to its capability to 
 model diverse problems in machine learning and statistics such as spam filtering, ad selection for search engines, 
 and recommender systems, to name a few. It is known that the online projected gradient descent 
 algorithm achieves a tight $O(\sqrt{T})$ regret 
 bound~\citep{zinkevich2003online}. However, in 
 many modern machine learning scenarios, one of the main computational bottlenecks is the 
 projection onto the constraint  set $\constraint$.  For example, in 
 recommender systems and matrix 
 completion, projections amount to expensive linear algebraic operations. 
 Similarly, projections onto matroid 
 polytopes with exponentially many linear inequalities are daunting tasks in general. This difficulty  has 
 motivated the use of projection-free algorithms 
 ~\citep{hazan2012projection,hazan2016introduction, Chen2018Projection} for which the most efficient 
 one achieves $O(T^{3/4})$ regret.

 In this paper, we consider a more difficult, and very often more realistic, OCO setting where the feedback is 
 incomplete.  More precisely, we consider a bandit feedback model where the only information observed by 
 the learner at iteration $t$ is the loss $f_t(\Bx_t)$  at the point $\Bx_t$  that she has chosen. In particular, the 
 learner does not know the loss had she chosen a different point $\Bx_t$. Therefore, the learner 
 has to 
 balance between exploiting the information that she has gathered and  
 exploring the new data. This 
 exploration-exploitation balance has been done beautifully by 
 \citep{Flaxman2005Online} to achieve 
 $O(T^{3/4})$ regret. With extra assumption on the loss functions (e.g., strong convexity), the 
 regret bound 
 has been recently improved to $ \tilde{O}(T^{1/2}) 
 $ ~\citep{Hazan2016optimal,Bubeck2015Bandit,Bubeck2017Kernel}. Again, all 
 these works either 
 rely on the 
 computationally expensive projection operations or inverting the Hessian 
 matrix of a 
 self-concordant barrier.  In 
 addition,   regret bounds usually have a very high polynomial dependency on 
 the dimension. 
 
 In this paper, we develop the first computationally efficient projection-free 
 algorithm with a sublinear 
 regret bound of $O(T^{4/5})$ 
 on the expected regret. We also show that the dependency on the dimension is linear. 
The regret bounds in different OCO settings  are summarized 
in~\cref{tab:regret_bound}.

\begin{table}[bht]
	\begin{threeparttable}[t]
		\centering
		\begin{tabular}{lll}
			\toprule
			& \textbf{Online} & \textbf{Bandit}	\\  \midrule
			\textbf{Projection} & $ O(T^{1/2}) 
			$\tnote{$\dagger$} & $ O(T^{3/4}) 
			$\tnote{$\ddagger$}, 
			$ \tilde{O}(T^{1/2}) 
			$\tnote{$\sharp$}
			\\ 
			\textbf{Projection-free} & $ O(T^{3/4}) 
			$\tnote{$\flat$} & 
			\cellcolor{gray!25}$ O(T^{4/5}) $ 
			(\textbf{this 
				work})\\ 
			\bottomrule
		\end{tabular}
		\begin{tablenotes}
			\item[$\dagger$] \cite{zinkevich2003online}
			\item[$\ddagger$] \cite{Flaxman2005Online}
			\item[$\sharp$] 
			\cite{Hazan2016optimal,Bubeck2015Bandit,Bubeck2017Kernel}
			\item[$\flat$] 
			\cite{hazan2012projection}\cite[Alg.~24]{hazan2016introduction}
		\end{tablenotes}
		\caption{Regret bounds in various settings of adversarial  online convex
			optimization.\label{tab:regret_bound}}
	\end{threeparttable}
\end{table}


\subsection*{Our Contributions}

\textbf{Sublinear regret with computational efficiency.} While there is a line 
of recent work that 
attains the minimax 
bound~\citep{Hazan2016optimal,Bubeck2015Bandit,Bubeck2017Kernel}, these 
algorithms have computationally expensive 
parts, such as inverting the Hessian of the self-concordant barrier. In 
contrast to these works that seek the lowest regret bound, we try to find a 
computationally efficient solution that attains a sublinear regret bound. 
Therefore, we have to avoid computationally expensive techniques like 
projection, Dikin ellipsoid and self-concordant barrier. 
As is shown in the experiments, our algorithm is simple and effective as  it 
only requires solving a linear optimization problem, while preserving a 
sublinear regret bound.

\textbf{Techniques.} The Frank-Wolfe (FW) algorithm may perform 
arbitrarily poorly with stochastic 
gradients even in the offline setting~\citep{hassani2017gradient}. Since the 
one-point estimator of gradient has a large variance, a simple combination of 
online FW~\citep{hazan2012projection} and one-point 
estimator~\citep{Flaxman2005Online} may not work. This is in fact shown empirically in 
Fig~\ref{fig:quad} when the loss functions are quadratic.    In addition, the online FW algorithm 
of \cite{hazan2012projection} 
is infeasible in the bandit setting. Basically, in each iteration of the online FW, the linear 
objective is the average gradient of all previous functions at a new point 
$x_t$. 
Note that in the bandit setting, it is impossible to evaluate the gradient of 
$f_i$ at 
$x_t$ ($i < 
t$), even with 
one-point estimators of \cite{Flaxman2005Online}. 

Our work has two major differences with \citep{hazan2012projection}. First, to 
make it a bandit 
algorithm, our linear objective is the sum of previously estimated gradients ($ 
\sum_{\tau=1}^{t-1} \Bg_\tau $, where $ \Bg_\tau $ is the one-point estimator 
of $ 
\nabla f_\tau(\Bx_\tau) $), rather than $ \sum_{\tau=1}^{t-1} \nabla 
f_\tau(\Bx_{t-1}) $. Second, we add a regularizer 
to 
stabilize the prediction. 

\section{Preliminaries}
\subsection{Notation} We 
let $ S^n \triangleq \{ \Bx\in \mathbb{R}^n: \| \Bx \|=1 \} $ and $ B^n 
\triangleq \{ \Bx\in \mathbb{R}^n: \| \Bx \|\le 1 \} $ denote the  unit sphere 
and  the  
unit ball in the $ n $-dimensional Euclidean space, respectively. Let $ \Bv $ 
be a random vector. We write $ \Bv\sim S^n 
$ and $ \Bv\sim B^n $ to indicate that $ \Bv $ is uniformly distributed over $ S^n $ and $ 
B^n $, respectively.

For any point set $ \CD \subseteq \mathbb{R}^n $ and $ \alpha>0 $, we denote $ 
\{ \Bx\in 
\mathbb{R}^n: \frac{1}{\alpha}\Bx\in \CD \} $ by $ \alpha \CD 
$.                         
Let $ f:\CD\to \mathbb{R} $ be a real-valued function on domain $ \CD 
\subseteq\mathbb{R}^n $. Its 
sup norm is given by $ \|f\|_{\infty} \triangleq \sup_{\Bx\in \CD}|f(\Bx)| $. 
We say that the function $ f:\CD\to \mathbb{R} $ is \emph{$ \alpha $-strongly 
convex}~\cite[pp.~63--64]{nesterov2003introductory} if $ f $ 
is continuously 
differentiable, $ \CD $ is a convex set, and the following inequality holds for 
$ \forall \Bx,\By \in \CD 
$ \[ 
f(\By)\geq (\Bx)+\nabla f(\Bx)^\top (\By-\Bx)+\frac{1}{2}\alpha \| \By-\Bx \|^2.
 \]
 An equivalent definition of strong convexity  is $(\nabla f(\Bx)-\nabla f(\By))^\top (\Bx-\By)\ge \alpha \| \Bx-\By 
 \|^2,$ for all 
 $\Bx,\By\in \CD $. 
%
We say that $ f $ is  \emph{$ G $-Lipschitz} if $ \forall 
  \Bx,\By\in \CD 
  $, $ \| f(\Bx)-f(\By) \| \leq G\| \Bx-\By \| $. 
   In this paper, we assume that the loss functions are all convex and bounded, 
   meaning that there 
  is a finite $M$ such that $\|f\|_{\infty}\leq M$. We also assume 
  that they are differentiable with uniformly bounded gradients, \emph{i.e.}, 
  there 
  exists a finite $G$ such that $\|\nabla 
  f\|_{\infty}\leq G$.

\subsection{Bandit Convex Optimization}
Online convex optimization is performed in a sequence of consecutive rounds,
where at round $t$, a learner has to choose an action $\Bx_t$ from a convex decision set
$\constraint\subseteq \mathbb{R}^n$. Then, an adversary  chooses a loss function $f_t$ from a family $ \CF 
$ 
of bounded convex 
functions. Once the action and the loss function are determined, the learner 
suffers a loss $f_t(\Bx_{t})$. The 
aim is to minimize regret  which is the gap 
between the accumulated  loss and the minimum loss in hindsight. More formally, the 
regret of a learning algorithm $ \CA $ after $T$ rounds is given by
\[ 
\CR_{\CA, T} \triangleq \sup_{ \{ f_1,\dots,f_T \} \subseteq \CF  } \left\{ 
\sum_{t=1}^{T} f_t(\Bx_t) 
- 
\min_{\Bx\in \CD} \sum_{t=1}^{T} f_t(\Bx) \right\}. 
\]
In the full information setting, the learner receives the loss function $f_t$ as a feedback (usually by having 
access to the gradient of $f_t$ at any feasible decision domain). In the bandit 
setting, however, the 
feedback  is limited to the  loss  at the point that she has 
chosen, \emph{i.e.}, $f_t(\Bx_t)$. In 
this 
paper, we consider the bandit setting where the family $\CF$ consists of bounded convex functions 
with  uniformly 
bounded gradients.  Under these conditions, we propose a projection-free  algorithm 
$\CA$ that achieves an expected regret of $\expect[\CR_{\CA, T}]=O(T^{4/5})$.

\subsection{Smoothing}\label{sub:smoothing}

A key ingredient of our solution relies on constructing the smoothed version of 
loss functions. Formally, 
for a function $ f $, its $ \delta $-smoothed version is defined by 
\begin{equation*}
\label{eq:smoothed function}
\hat{f}_\delta(\Bx) = \expect_{\Bv \sim B^n}[f(\Bx + \delta \Bv)],
\end{equation*} 
where $\Bv$ is drawn uniformly at 
random from the $ n $-dimensional unit ball $B^n$.  Here, $ \delta $ controls the 
radius of the ball that the function $ f $ is averaged over. Since $ \hat{f}_\delta $ is a smoothed 
version of $f$, it  inherits analytical properties from $ f $.  \cref{lem:smoothing} 
formalizes this idea. 

\begin{lemma}[Lemma 2.6 in \citep{hazan2016introduction}] \label{lem:smoothing}
	Let $ f:\CD\subseteq \mathbb{R}^n\to \mathbb{R} $ be a convex, $ G 
	$-Lipschitz continuous function and let $ \CD_0\subseteq \CD $ be such that 
	$ \forall \Bx\in \CD_0, \Bv\in S^n $, $ \Bx+\delta \Bv\in \CD $. Let 
	$\hat{f}_\delta$ be the $\delta$-smoothed function defined above. Then  $\hat{f}_\delta$ 
	is also convex, and $ \| 
	\hat{f}_\delta-f\|_{\infty} \le \delta G $ on $ \CD_0 $. 
\end{lemma}
Since $ \hat{f}_\delta $ is an approximation of $ f $, if one finds a minimizer 
of $ \hat{f}_\delta $,   \cref{lem:smoothing} implies that it also minimizes $ 
f $ approximately. Another advantage of considering the smoothed version is 
that it admits 
one-point gradient estimates of $\hat{f}_\delta$ based on samples of $f$. This idea was first 
introduced 
in~\citep{Flaxman2005Online} for developing an online gradient descent 
algorithm without having 
access to gradients.

\begin{lemma}[Lemma 6.4 in 
\citep{hazan2016introduction}]\label{lem:smooth_estimate}
Let $\delta > 0$ be any fixed positive real number and $\hat{f}_\delta$ be the 
$\delta$-smoothed version of  function $ f $. The following equation holds
\begin{equation}
\nabla 
\hat{f}_\delta (\Bx) =  \expect_{\Bu \sim S^n}\left[\frac{n}{\delta}f(\Bx 
+ 
\delta \Bu)\Bu\right].\label{eq:point_estimate}
\end{equation}
\end{lemma}
 \cref{lem:smooth_estimate} suggests that in order to sample the gradient of $ 
 \hat{f}_\delta $ at a point $ \Bx $, it suffices to evaluate $ f $ at a random 
 point $ \Bx+\delta \Bu $ around the point $ \Bx $. 
%




\section{Algorithms and Main Results}\label{sec:algorithm}

The first key idea of our proposed algorithm is to construct a 
follow-the-regularized-leader objective
\begin{equation}\label{eq:FTRL}
F_t(\Bx) = \eta \sum_{\tau=1}^{t-1} 
\nabla f_{\tau}(\Bx_{\tau})^\top \Bx+\| \Bx-\Bx_1 \|^2.
\end{equation}
Instead of minimizing $ F_t $ directly
 (as it is done in 
follow-the-regularized-leader 
algorithm), the learner first solves a linear program over the decision set 
$\constraint$ 
\begin{equation}\label{eq:lp}
\Bv_t=\min_{\Bx\in 
	\constraint} \{ \nabla F_t(\Bx_t)\cdot \Bx \},
\end{equation}
 and then updates its decision as follows
 \begin{equation}\label{eq:update}
 	\Bx_{t+1}\gets (1-\sigma_t)\Bx_t+\sigma_t \Bv_t.
 \end{equation}
Note that minimizing $ F_t $ requires solving a quadratic optimization problem, 
which is as computationally prohibitive as a projection operation. In contrast, since the update 
in~\cref{eq:update} is a convex  combination between $\Bv_t$ 
and 
$\Bx_{t}$, the iterates always lie inside the convex decision set $\constraint$, thus no projection is 
needed. This is the main idea 
behind the  online conditional gradient algorithm (Algorithm~24 
in~\citep{hazan2016introduction}).  In the bandit setting (the focus of this 
paper), the gradients $ 
\nabla 
f_{\tau}(\Bx_{\tau}) $ are unavailable, hence the learner cannot perform steps \eqref{eq:FTRL} and 
\eqref{eq:lp}.  To tackle this 
issue, we introduce the second ingredient of our algorithm, namely, the smoothing and 
one-point gradient estimates~\citep{Flaxman2005Online}. Formally, at the $ t 
$-th iteration, 
rather than selecting $ \Bx_t $, the learner  plays a random point $ \By_t $ that is $ \delta 
$-close to 
$ \Bx_t $ and in return observes the cost $ f_t(\By_t) $.  As shown in Lemma 
\ref{lem:smooth_estimate}, $ f_t(\By_t) $ can be used to construct an unbiased 
estimate $ \Bg_t $ for 
the gradient of 
the $\delta$-smoothed version of $ f_t $ at point $ \Bx_t $, \emph{i.e.}, 
$\expect [\Bg_t] = \nabla \hat{f}_{t,\delta}(\Bx_t)$, where 
$\hat{f}_{t,\delta}(\Bx_t) \triangleq \expect_{\Bv \sim B^n}[f_t(\Bx_t + \delta 
\Bv)]$. This observation suggests 
that we can replace 
$ \nabla f_t(\Bx_t) $ by  $ \Bg_t $ in the 
follow-the-regularized-leader objective \eqref{eq:FTRL} to obtain a variant that relies on the 
one-point gradient estimate, \emph{i.e.}, 
\begin{equation}\label{eq:FTRLB}
F_t(\Bx) = \eta \sum_{\tau=1}^{t-1}\Bg_{\tau}^\top 
\Bx + \|\Bx - \Bx_1\|^2.
\end{equation}
Note that forming $F_t(\Bx)$ in \eqref{eq:FTRLB} is fully realizable for a learner in a bandit 
setting. The full description of our algorithm is outlined  in  \cref{alg:conditional_bandit}. Even 
though the  objective function $F_t(\Bx)$  relies on the unbiased estimates of the smoothed 
versions of $f_t$ (rather than $f_t$ itself), it is not far  off from the original  objective 
(shown in \cref{eq:FTRL}) if the distance between the random 
point $ \By_t $ and the point $ \Bx_t $ is properly chosen. Therefore, minimizing the sum of 
smoothed versions of $f_t$ (as it is done by \cref{alg:conditional_bandit}) will end up 
minimizing the actual regret. This intuition is formally proven in 
\cref{thm:regret_bound}. 
Without loss of generality, we assume additionally that the 
constraint $ \constraint $ contains a ball of radius $ r $ centered at the 
origin (this is always achievable by shrinking the constraint set as long as it 
has a non-empty interior).

\begin{algorithm}[htb]
	\begin{algorithmic}[1]
		\Require horizon $ T $, constraint set $ \constraint $
		\Ensure $ \By_1, \By_2, \dots, \By_T $
		\State $ \Bx_1 \in (1-\alpha) \constraint $
		\For{$ t=1,\ldots,T $}
		\State $ \By_t\gets \Bx_t+\delta \Bu_t $, where $ \Bu_t \sim S^n $ 
		\State Play $ \By_t $ and observe $ f_t(\By_t) $ 	
		\State $ \Bg_t\gets \frac{n}{\delta} f_t(\By_t)  \Bu_t$  
		\commentcolor{$ 
			\Bg_t $ is an unbiased estimator of $ \nabla 
			\hat{f}_{t,\delta}(\Bx_t) 
			$}
		\State $ F_{t}(\Bx) \gets \eta \sum_{\tau=1}^{t-1}\Bg_{\tau}^\top 
		\Bx + \|\Bx - \Bx_1\|^2 $ \label{ln:F_t}
		\State $ \Bv_{t}\gets \argmin_{\Bx\in (1-\alpha) \constraint} \{ \nabla 
		F_t(\Bx_t) \cdot \Bx\} $\commentcolor{Solve a linear optimization 
		problem}	\label{line:linear_optimization}	
		\State $ \Bx_{t+1}\gets (1-\sigma_t) \Bx_t +\sigma_t \Bv_t $
		\EndFor
	\end{algorithmic}
	\caption{Projection-Free Bandit Convex 
	Optimization}\label{alg:conditional_bandit}
\end{algorithm}

\begin{theorem}[\textbf{Proof in \cref{sec:proof}}]\label{thm:regret_bound}
	Assume that for every $ t\in \mathbb{N}_{\geq 1} $, $ f_t $ is convex, $   
	\|f_t\|_{\infty} \leq M $ 
	on $ \constraint $, $ 	
	\sup_{\Bx\in \constraint} \|\nabla 
	f_t(\Bx)\|\leq G $, $ r B^n\subseteq \constraint \subseteq RB^n  $, and that
	the 
	diameter of $ \constraint $ is $ D<\infty $. If we  set $ \eta = 
	\frac{D}{\sqrt{2}nM}T^{-4/5} $, $ \sigma_t = t^{-2/5} $,  $ \delta = c
	T^{-1/5} $, and $ \alpha=\delta/r < 1$ in \cref{alg:conditional_bandit}, 
	where $c>0$ is a  constant, we have $ \By_t \in \constraint,  \forall 1\le t\le T$. Moreover, 
	the expected regret $ \expect[\CR_{\CA, T}] $ up to horizon $ T $  is at most
	\[ 
\frac{\sqrt{2}nMD}{c^2}T^{3/5} + 
(\sqrt{2}nMD+ \frac{5\sqrt{2}}{4}DG+3cG+cRG/r)T^{4/5}.\\
	 \]
\end{theorem}
Note that the regret bound of \cref{alg:conditional_bandit} depends linearly on the dimension $n$.






A minor drawback of \cref{alg:conditional_bandit} is that it requires the knowledge 
of the horizon $ T $. This problem can be easily circumvented  via the doubling trick while 
preserving 
the regret bound of \cref{thm:regret_bound}. 
%
The doubling trick was first proposed in~\citep{auer1995gambling} and its key 
idea is to invoke the base algorithm repeatedly with a doubling horizon. \cref{alg:anytime} outlines 
an \textit{anytime} algorithm for BCO using the doubling trick. \cref{thm:anytime} shows that for 
any $ t\geq 1 $, the 
expected
regret of \cref{alg:anytime} by the end of the $ t $-th iteration is bounded by 
$ O(t^{4/5}) $.

\begin{algorithm}[htb]
	\begin{algorithmic}[1]
		\Require constraint set $ \constraint $
		\Ensure $ \By_1, \By_2, \dots $
		\For{$ m=0,1,2,\ldots $}
		\State Run \cref{alg:conditional_bandit} with horizon $ 2^m $ from the 
		$ 2^m $-th iteration (inclusive) to the $ (2^{m+1}-1) $-th iteration 
		(inclusive). 
		\State Let $ \By_{2^m},\dots, \By_{2^{m+1}-1} $ be the points that 
		\cref{alg:conditional_bandit} selects for the objectives $ 
		f_{2^m},\dots, f_{2^{m+1}-1} $.
		\EndFor
	\end{algorithmic}
	\caption{Anytime Projection-Free  
		Bandit Convex Optimization}\label{alg:anytime}
\end{algorithm}

\begin{theorem}[\textbf{Proof in \cref{app:anytime}}]\label{thm:anytime}
	If the regret bound of \cref{alg:conditional_bandit} for horizon $ T $ is $ 
	\beta T^{4/5} $, then for any $ t\geq 1 $, the expected regret of 
	\cref{alg:anytime} by the end of the $ t $-th iteration is at most $$ 
	\expect[\CR_{\CA, T}]=\frac{\beta}{1-2^{-4/5}} (t+1)^{4/5} = O(t^{4/5}) .$$
\end{theorem}

\begin{figure*}[htb]
	\centering
	\begin{subfigure}[t]{0.238\textwidth}
		\includegraphics[width=\textwidth]{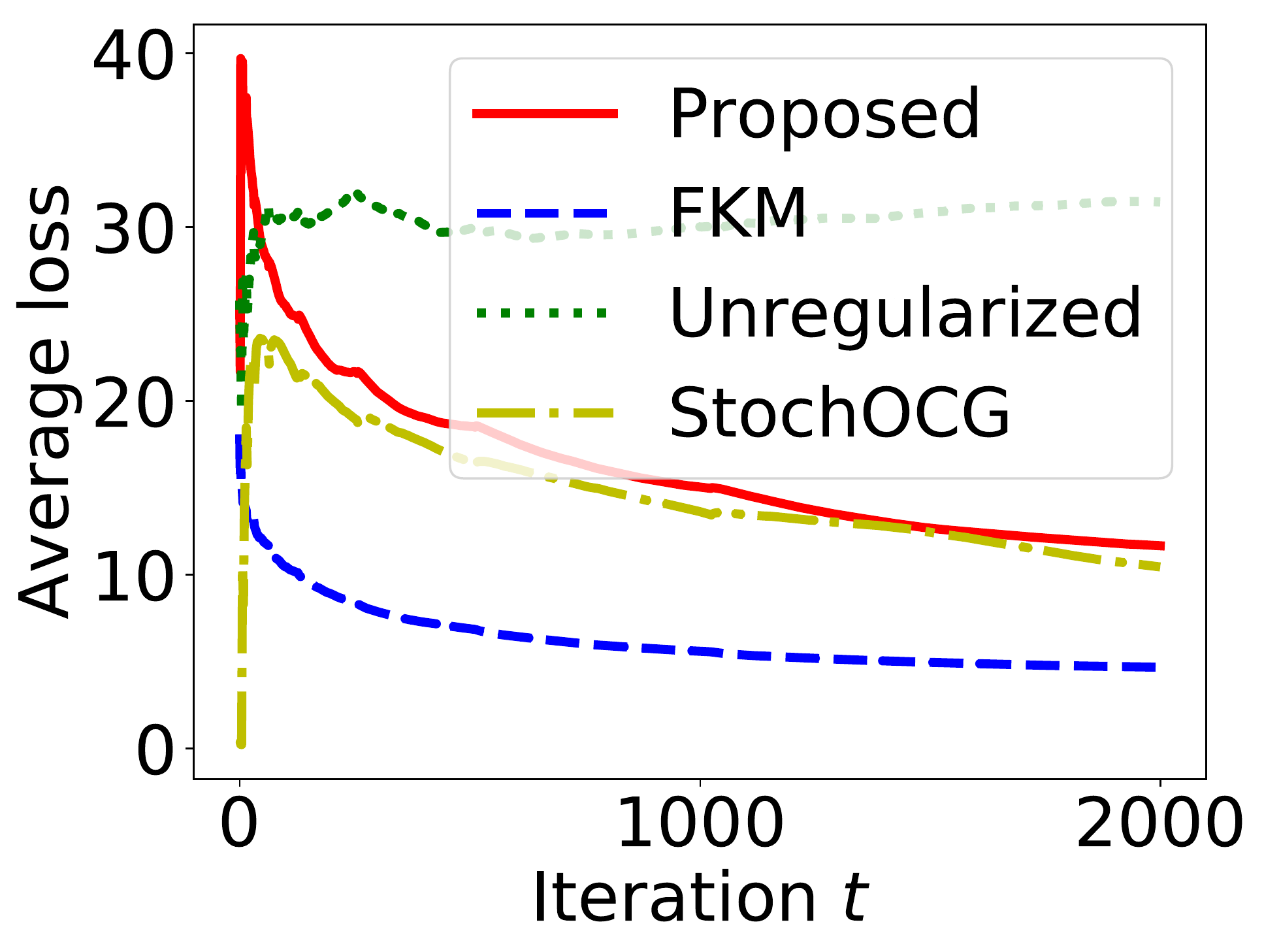}
		\caption{Quadratic programming}
		\label{fig:quad}
	\end{subfigure}
	\begin{subfigure}[t]{0.25\textwidth}
		\includegraphics[width=\textwidth]{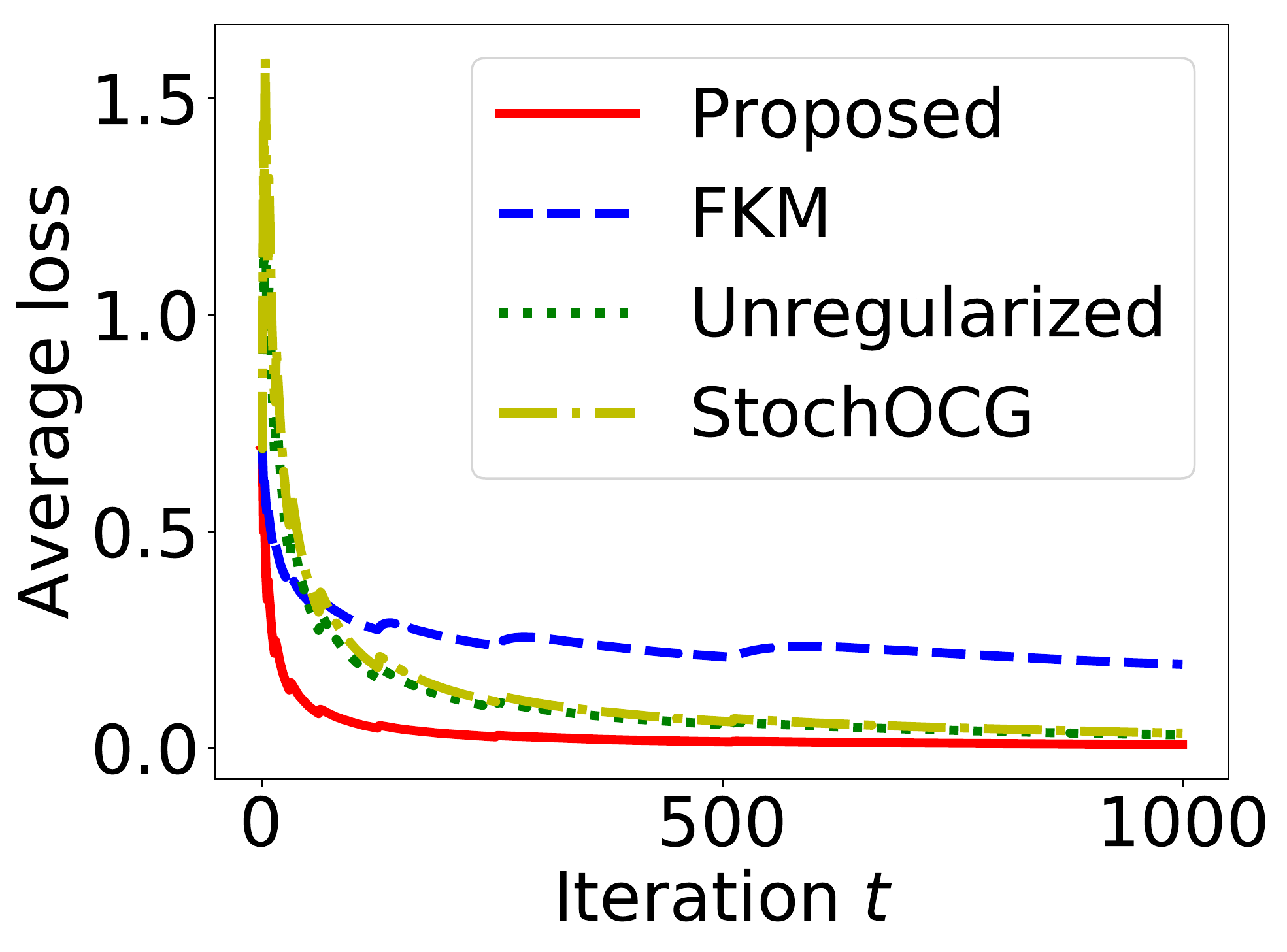}
		\caption{Portfolio selection}
		\label{fig:portfolio}
	\end{subfigure}
\begin{subfigure}[t]{0.252\textwidth}
	\includegraphics[width=\textwidth]{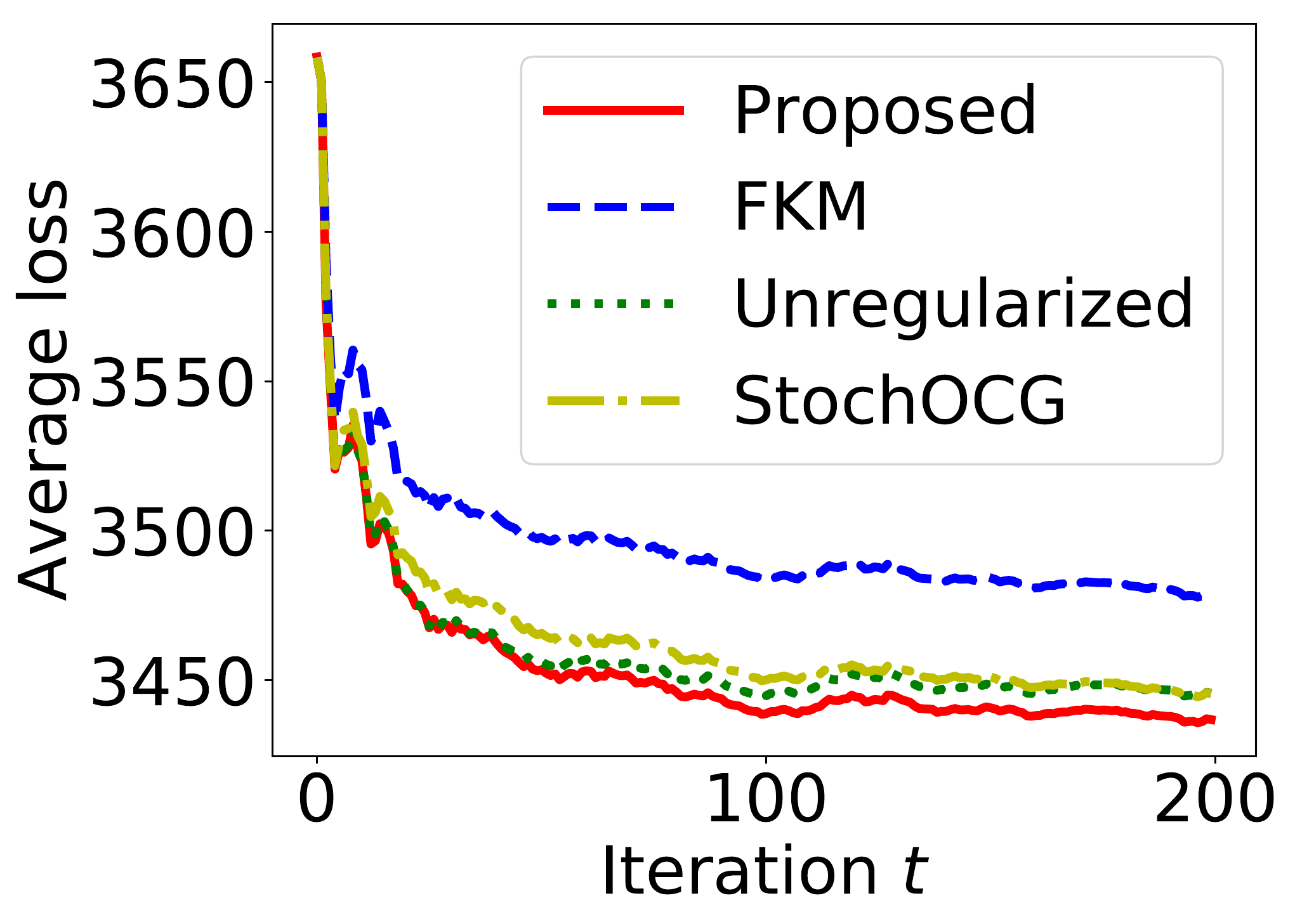}
	\caption{Matrix completion}
	\label{fig:completion}
\end{subfigure}
\begin{subfigure}[t]{0.246\textwidth}
	\includegraphics[width=\textwidth]{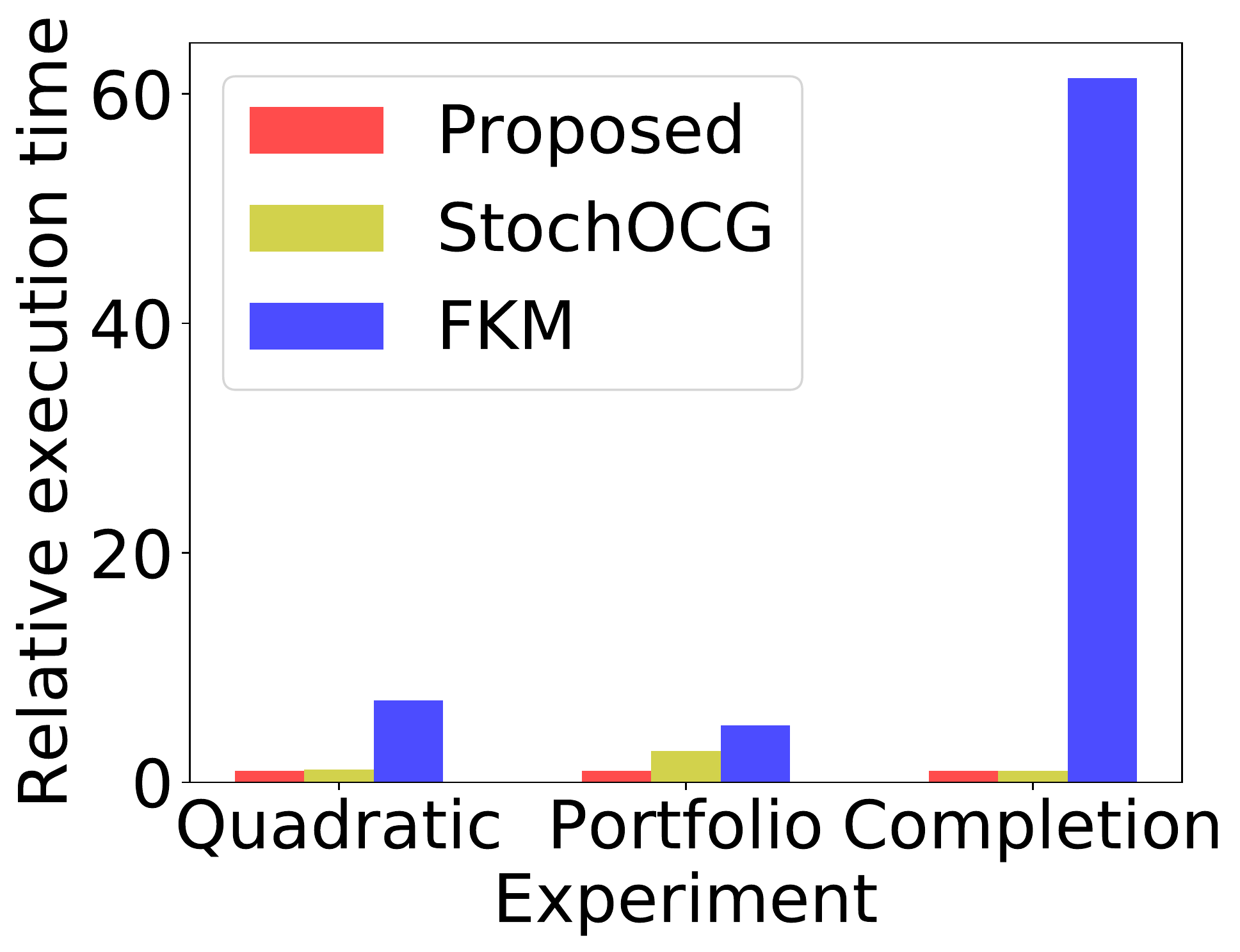}
	\caption{Relative execution time}
	\label{fig:runtime}
\end{subfigure}
	\caption{In~\cref{fig:quad,fig:portfolio,fig:completion}, we 
	show the 
	average loss versus the number of iterations in the three sets of 
	experiments. The relative execution time is shown in~\cref{fig:runtime}, 
	where the execution time of the proposed algorithm is set to $ 1 $.
}\label{fig:experiment_results}
\end{figure*}
\section{Experiments}
In our set of experiments, we compare \cref{alg:anytime} with the following baselines:
\begin{itemize}
	\item  \AlgFKM: Online projected gradient descent with  spherical gradient 
	estimators  
	\citep{Flaxman2005Online}. 
	\item \AlgUn: A variant of our proposed algorithm 
	without the regularizer $ \| \Bx-\Bx_1 \|^2 $ in line~\ref{ln:F_t} of 
	\cref{alg:conditional_bandit}.
	\item \AlgOCG:  Online conditional gradient \citep{hazan2016introduction} 
	with stochastic 
	gradients (not a bandit algorithm). Such stochastic gradients are formed  by 
	adding Gaussian noise with standard deviation $ n $ to the exact gradients. 
\end{itemize}
The anytime version of the algorithms (obtained via the doubling trick) is 
used. Therefore the horizon $ T $ is unknown to the algorithms.
Note that the standard deviation of 
the point estimate used in \AlgFKM and our proposed method  is proportional to 
the dimension $ n $. This is why the 
standard deviation of the Gaussian noise in \AlgOCG is set to $n$ to make the noise in the gradients 
comparable. 

We performed three sets of experiments in total. In all of them we report the 
average loss defined as 
$\expect[\sum_{t=1}^{T} f_t(\Bx_t)]/T$. 

\textbf{Quadratic programming:}
In the first 
experiment, the loss functions are quadratic, \emph{i.e.}, $ f_t(\Bx) = 
\frac{1}{2}\Bx^\top \BG_t^\top \BG_t \Bx + \Bw_t^\top \Bx $. Each entry of $ 
\BG_t $ 
and $ \Bw_t $ is sampled from the standard normal distribution. The convex 
constraint of this problem is a polytope $ \{ \Bx:\bm{0}\leq \Bx\leq \bm{1}, 
\BA \Bx \leq \bm{1} \} $ and each entry of $ \BA $ is sampled from the uniform 
distribution on $ [0,1] $. The average loss is illustrated 
in~\cref{fig:quad}. We observe that the average loss of our proposed 
algorithm declines as the number of iterations increases. This agrees with the theoretical 
sublinear regret bound. \AlgOCG  has a similar performance while \AlgFKM 
exhibits the lowest 
loss.  In contrast, the 
loss of \AlgUn 
appears to be linear which shows the significance of regularization to achieve low regret. This observation also 
suggests that simply combining 
\citep{hazan2012projection} and smoothing may not work in practice.


\textbf{Portfolio selection:} For this experiment, we 
randomly select $ n=100 $ stocks from Standard \& Poor's 500 index component 
stocks and consider their prices during the business days between February 
18th, 2013 and November 27th, 2017. We follow the formulation in~\cite[Section 
1.2]{hazan2016introduction}. Let $ \Br_t\in \mathbb{R}^n $ be a vector 
such that $ \Br_t(i)  $ is the ratio of the price of stock $ i $ on day $ t+1 $ 
to its price on day $ t $. An investor is trying maximize her wealth by investing on different stock options. 
If 
$W_t$ denotes her wealth on day $t$, then we have the following recursion: $W_{t+1} = W_t \cdot  
\Br_t^\top 
\Bx_t $. After $T$ days of investments, the total wealth will be $W_T = W_1\cdot \prod_{t=1}^{T} 
\Br_t^\top 
\Bx_t $. To maximize  the wealth, the investor has to maximize $ \sum_{t=1}^{T} \log(\Br_t^\top 
\Bx_t) $, or equivalently minimize its negation. Thus, we can define $ 
f_t(\Bx_t)\triangleq 
-\log(\Br_t^\top 
\Bx_t) $. \AlgFKM requires that  the constraint set contains the unit ball. To this end, we set $ 
\By_t = 
2n\Bx_t-1 $ 
so that $ 
\By_t $ lies in an enlarged region $ \Delta'_n \triangleq \{ \By\in 
\mathbb{R}^n: -1\leq \By(i) \leq 2n-1, \sum_{i=1}^{n} \By(i)\leq n \} $. In 
addition, the objective functions $ f_t $ are viewed as functions of $ \By_t $ 
rather 
than $ \Bx_t $. The average losses versus the number of 
iterations 
are presented in~\cref{fig:portfolio}. Our proposed algorithm has the lowest 
loss in this set of experiments while \AlgFKM has the largest.

\textbf{Matrix completion:} Let $ \{\BM_t\}_{t=1}^T $ be symmetric positive 
semi-definite (PSD)
matrices, where $ \BM_t = \BN_t^\top \BN_t$ and every entry of $ \BN_t\in 
\mathbb{R}^{k\times n} $ obeys the standard normal distribution. At each 
iteration, half of the entries of $ \BM_t $ are observed. We set $ n=20 $ and $ 
k=18 $. We denote 
the 
entries of $ \BM_t $ disclosed at the $ t $-th iteration by $ O_t $. We want to 
minimize $ f_t(\BX_t) \triangleq\frac{1}{2}\sum_{(i,j)\in 
O_t} (\BX_t[i,j]-\BM_t[i,j])^2 $ subject to $ \|\BX_t\|_*\le k $, where $ \BX_t 
$ 
is 
of the same shape as $ \BM_t $ and $ 
\|\cdot\|_* $ denotes the nuclear norm. The nuclear norm constraint is a 
standard convex relaxation of the rank constraint $ \rank(\BX)\le k $. 
The linear optimization step in \cref{line:linear_optimization} of 
\cref{alg:conditional_bandit} has a closed-form solution $ \Bv_t = k 
\Bv_{\max} \Bv_{\max}^\top$, where $ \Bv_{\max}  $ is the eigenvector of the
largest eigenvalue of  $ -\nabla f_t(\BX_t) $ 
\cite[Section~7.3.1]{hazan2016introduction}. The largest eigenvector can be 
computed very efficiently using power iterations, whilst it is extremely costly 
to perform projection onto a convex subset of the space of PSD matrices. 
As shown in 
\cref{fig:runtime}, the efficiency of the proposed algorithm is $ 61 $ times 
that of the projection-based \AlgFKM algorithm.
The 
average loss 
of the 
algorithms is 
shown in \cref{fig:completion}. Our proposed algorithm outperforms the other 
baselines  while  \AlgFKM suffers the largest loss. 


 
 We also observe rises of the curves at their initial stage in 
 \cref{fig:experiment_results}. They are due to the doubling trick 
 (\cref{alg:anytime}) and 
a 
small denominator of the average loss. The unknown horizon is divided into 
epochs with a doubling size (1, 2, 4, and so forth). When the algorithm starts 
a new epoch, everything is reset and the algorithm learns from scratch. 
Furthermore, the denominator of the average loss is small (it is initially 1, 
and 
then becomes 2, 3, 4, and so forth) at the initial stage. Therefore, due to the 
frequent resets and a small denominator, the behavior is less stable. As the 
epoch size and denominator grow, the average loss declines steadily.

The execution time is shown 
in~\cref{fig:runtime}. It was measured on eight Intel Xeon E5-2660 V2 cores and 
the algorithms were implemented in  Julia. 50 repeated experiments 
were run in parallel.
It can be 
observed that our proposed algorithm 
runs significantly faster than the \AlgFKM algorithm (mostly by avoiding the 
projection steps). 
Specifically, its efficiency 
is almost 7 times, 5 times, and 61 times that of the FKM 
algorithm in 
the 
three sets of experiments, respectively. \AlgOCG requires computation of 
gradients and is also slower than the proposed algorithm.
\section{Proof of \cref{thm:regret_bound}}\label{sec:proof}
First we show $ \By_t\in \constraint$. Since $\Bv_t \in (1-\alpha) 
\constraint$, $ \Bx_1 \in (1-\alpha) \constraint$ and $\Bx_{t+1} = 
(1-\sigma_{t})\Bx_{t}+\sigma_{t}\Bv_{t}$, by induction and the convexity of 
$\constraint$, we have $\Bx_{t} \in 
(1-\alpha) \constraint$ for every $ t $.  
Recall that $ 
\By_t = \Bx_t + \delta \Bu_t$, where $\Bu_t 
\in S^n$ and $\alpha = \delta/r$. Since $\constraint$ is convex and $rS^n 
\subseteq rB^n \subseteq \constraint$, we have $\By_t \in (1-\alpha)\constraint 
+ \alpha rS^n \subseteq (1-\alpha)\constraint + \alpha \constraint = 
\constraint.$

	Let $ \Bx_t^* \triangleq \argmin_{\Bx\in (1-\alpha)\constraint} F_t(\Bx) $ 
	and 
	$\hat{f}_{t,\delta}(\Bx_t) \triangleq \expect_{\Bv \sim B^n}[f_t(\Bx_t + 
	\delta 
	\Bv)]$. The first step is to 
	derive a bound on $ \sum_{t=1}^{T} \Bg_t^{\top} (\Bx_t^* - \Bz)$.
	We need the following 
	lemma. 
	\begin{lemma}[Lemma 2.3 in \citep{shalev2012online}]\label{lem:regret_bound}
		Let $ \Bw_1,\Bw_2,\ldots $ be a sequence of vectors in $ 
		(1-\alpha)\constraint $ 
		such 
		that $ \forall t, \Bw_t = \argmin_{\Bw\in (1-\alpha)\constraint} 
		\sum_{i=1}^{t-1} 
		f_i(\Bw)+R(\Bw) $.
		Then for every $ \Bz\in (1-\alpha)\constraint $, we have
		$
		\sum_{t=1}^{T}(f_t(\Bw_t)-f_t(\Bz)) \leq R(\Bz) - 
		R(\Bw_1)+\sum_{t=1}^{T}(f_t(\Bw_t)-f_t(\Bw_{t+1})).
		$
	\end{lemma}
	By 
	\cref{lem:regret_bound} and in light of the fact that $ \Bx_1^*=\Bx_1 $, $ 
	\forall \Bz\in (1-\alpha)\constraint $, we 
	have
	\begin{dmath}\label{eq:bound_from_lemma}
	\sum_{t=1}^{T} \Bg_t^{\top} (\Bx_t^* - \Bz)  \leq \|\Bz - \Bx_1\|^2/\eta - 
	\|\Bx_1^*-\Bx_1\|^2/\eta  
	  + \sum_{t=1}^{T} \Bg_t^\top(\Bx_t^*-\Bx_{t+1}^*)
	 =  \|\Bz - \Bx_1\|^2/\eta 
	+ \sum_{t=1}^{T} \Bg_t^\top(\Bx_t^*-\Bx_{t+1}^*).
	\end{dmath}
	Let $ \CF_t $ be the $ \sigma $-field generated by $ 
	\Bx_1,\Bg_1,\Bx_2,\Bg_2,\ldots, \Bx_{t-1}, \Bg_{t-1}, \Bx_t $. Note that $ 
	\Bx_t^* $ is a function of $ \Bg_1, \dots, \Bg_{t-1} $ and thus measurable 
	with respect to $ \CF_t $. Therefore we have $ \expect[\Bg_t^{\top} 
	(\Bx_t^* - 
	\Bz)] = \expect[\expect[\Bg_t^{\top} (\Bx_t^* - \Bz)|\CF_t]]  = 
	\expect[\expect[\Bg_t|\CF_t]^\top (\Bx_t^* - \Bz) ] = \expect [\nabla 
	\hat{f}_{t,\delta}(\Bx_t)^\top (\Bx_t^* - \Bz)]  .$
	To bound the second term 
	on the right-hand side of \cref{eq:bound_from_lemma}, note that $\Bg_t^\top 
	(\Bx_t^*-\Bx_{t+1}^*) \leq 2\eta \|\Bg_t\|^{2}$ (we will show it in 	
	\cref{app:Bregman}).
	Therefore we have
	$	\sum_{t=1}^{T} \Bg_t^\top(\Bx_t^*-\Bx_{t+1}^*)\leq 2\eta \sum_{t=1}^{T} \| 
	\Bg_t \|^2 \leq 2\eta n^2 M^2T /\delta^2$.
	Combining it with
	\cref{eq:bound_from_lemma}, we deduce 
	$
	\sum_{t=1}^{T} \Bg_t^{\top} (\Bx_t^* - \Bz) \le D^2/\eta + 2\eta n^2M^2 
	T/\delta^2$.
	Since 
	\begin{dmath}
	\label{eq:regret decomposition}
	\sum_{t=1}^{T}\expect[f_t(\By_t)-f_t(\Bz)] = 
	\sum_{t=1}^{T}\expect[f_t(\By_t)-f_t(\Bx_t)] 
	  +
	\sum_{t=1}^{T}\expect[f_t(\Bx_t)-f_t(\Bz)],
	\end{dmath}
	and the norm of the gradient of $ f_t $ is assumed to be at most $ G $
	\begin{dmath}
	\label{eq:regret1}
	\sum_{t=1}^{T}\expect[f_t(\By_t)-f_t(\Bx_t)] = 
	\sum_{t=1}^{T}\expect[f_t(\Bx_t + \delta \Bu_t)-f_t(\Bx_t)] 
	\leq \delta TG,
	\end{dmath}
	we only need to obtain an upper bound of 
	the second term on the right hand side of \cref{eq:regret decomposition}, 
	which is
	\begin{equation*} 	\begin{split}
	&\sum_{t=1}^{T}\expect[f_t(\Bx_t)-f_t(\Bz)] \\
	=& \expect[ \sum_{t=1}^{T}( 
	\hat{f}_{t,\delta}(\Bx_t)-\hat{f}_{t,\delta}(\Bz )) 
	+\sum_{t=1}^T (f_t(\Bx_t)-\hat{f}_{t,\delta}(\Bx_t)) \\
	& \qquad -\sum_{t=1}^T 
	(f_t(\Bz)-\hat{f}_{t,\delta}(\Bz)) ] \\
	\stackrel{(a)}{\leq}& \expect\left[ \sum_{t=1}^{T}( 
	\hat{f}_{t,\delta}(\Bx_t) 
	-\hat{f}_{t,\delta}(\Bz )) \right]+ 2\delta GT \\
	\stackrel{(b)}{\leq}& 
	\sum_{t=1}^{T} 
	\expect[ \nabla 
	\hat{f}_{t,\delta}(\Bx_t)^\top 
	(\Bx_t-\Bz)]+2\delta GT.
	\end{split}
	\end{equation*}
	Inequality $ (a) $ is due to 
	\cref{lem:smoothing}. We used the convexity of  $\hat{f}_{t,\delta}$ in $ 
	(b) $. We 
	split $  \nabla 
	\hat{f}_{t,\delta}(\Bx_t)^\top 
	(\Bx_t-\Bz) $ into $  \nabla \hat{f}_{t,\delta}(\Bx_t)^\top 
	(\Bx^*_t-\Bz)+   \nabla \hat{f}_{t,\delta}(\Bx_t)^\top 
	(\Bx_t-  \Bx^*_t) $ and thus obtain
	\begin{equation}
	\begin{split}
	\label{eq:bound_regret}
&\sum_{t=1}^{T}\expect[f_t(\Bx_t)-f_t(\Bz)]	\\
 \le & \sum_{t=1}^{T} \expect[ \nabla 
\hat{f}_{t,\delta}(\Bx_t)^\top 
	(\Bx^*_t-\Bz)]\\
	& \qquad + \sum_{t=1}^{T} \expect[ \nabla 
	\hat{f}_{t,\delta}(\Bx_t)^\top 
	(\Bx_t-  \Bx^*_t)]+2\delta GT\\
	=&\sum_{t=1}^{T} \expect[\Bg_t^{\top} 
	(\Bx_t^* - 
	\Bz)] + \sum_{t=1}^{T} \expect[ \nabla \hat{f}_{t,\delta}(\Bx_t)^\top 
	(\Bx_t-  \Bx^*_t)]\\
	& \qquad +2\delta GT\\
	 \le& D^2/\eta + 2\eta n^2M^2 
	T/\delta^2 + \sum_{t=1}^{T} \expect[ \nabla \hat{f}_{t,\delta}(\Bx_t)^\top 
	(\Bx_t-  \Bx^*_t)]\\
	&\qquad +2\delta GT.
	\end{split} 
	\end{equation}
	The next step is to bound $\nabla \hat{f}_{t,\delta}(\Bx_t)^\top 
	(\Bx_t-  \Bx^*_t)$. To this end, we need an auxiliary inequality as stated 
	in~\cref{lem:auxillary_inequality}.
	\begin{lemma}\label{lem:auxillary_inequality}
		The inequality  $
		-4 t^{2/5} (t+1)^{2/5}+4 t^{4/5}-2 t^{1/5} (t+1)^{1/5}+3 (t+1)^{2/5} 
		\geq 0
		$ holds for any $ t= 1,2,3,\dots$.
	\end{lemma}
	
	\begin{proof}
		We verify the inequality when $ t=1 $ or $ 2 $.
		When $ t\geq 3 $, we have \[ 
		(1+1/t)^{2/5} \geq 1 \geq \frac{8}{5}t^{-3/5}.
		\]
		Since $ 2(1+1/t)^{2/5}\geq 2(1+1/t)^{1/5} $, we obtain\[ 
		3(1+1/t)^{2/5} \geq 2(1+1/t)^{1/5} + \frac{8}{5}t^{-3/5}.
		\]
		Therefore, we have \begin{equation}\label{eq:aux_ineq1}
		3(1+1/t)^{2/5} - 2(1+1/t)^{1/5} - \frac{8}{5}t^{-3/5} \geq 0.
		\end{equation}
		Let $ g(t)=t^{2/5} $. Since $ g(t) $ is concave, we have $ 
		g(t+1)-g(t)\leq 
		g'(t) $, which gives
	$	(t+1)^{2/5}-t^{2/5}\leq \frac{2}{5}t^{-3/5}.$
		Combining the above inequality with \cref{eq:aux_ineq1}, we see \[ 
		3(1+1/t)^{2/5} - 2(1+1/t)^{1/5} + 4t^{2/5}-4(t+1)^{2/5} \geq 0.
		\]
		Multiplying both sides with $ t^{2/5} $, we complete the proof.
	\end{proof}
	In light of the inequality, we have
	\[
	\begin{split}
	&t^{3/5} (t+1)^{1/5} \left(\frac{3}{2 
		t^{4/5}}-\frac{2}{t^{2/5}}+\frac{2}{(t+1)^{2/5}}\right)\\
	= &
	\frac{-4 t^{2/5} (t+1)^{2/5}+4 t^{4/5}+3 (t+1)^{2/5}}{2 t^{1/5} 
		(t+1)^{1/5}} \\
	\geq & 1.
	\end{split}
	\]
	By algebraic manipulation,  we see
	\begin{equation} 
	\label{eq:ge_t_3_5}
	\begin{split}
	&\frac{2\sigma_{t+1}-2\sigma_t+(3/2)\sigma_t^2}{\sqrt{2\sigma_{t+1}}} \\
	=& 
	\frac{1}{\sqrt{2}}(t+1)^{1/5} \left(\frac{3}{2 
		t^{4/5}}-\frac{2}{t^{2/5}}+\frac{2}{(t+1)^{2/5}}\right)\\
	\geq & 
	\frac{1}{\sqrt{2}}t^{-3/5}.
	\end{split}
	\end{equation}
	If $ 1\le t\le T $, we deduce 
	\begin{equation}\label{eq:t_3_5_ge}
	\frac{1}{\sqrt{2}}t^{-3/5} \ge \frac{1}{\sqrt{2}}T^{-3/5} = \frac{\eta n 
		M}{\delta D} \geq \frac{\eta}{D} \| \Bg_s \|,\quad \forall 1\leq s\leq 
	T.
	\end{equation}
	Combining \cref{eq:ge_t_3_5} and \cref{eq:t_3_5_ge}, we deduce $$
	\eta\leq D \frac{2\sigma_{t+1}-2\sigma_t+(3/2)\sigma_t^2}{ \| \Bg_{t+1} \| 
		\sqrt{2\sigma_{t+1}}}, \quad \forall 1\leq t\leq T.
	$$
	The above inequality is equivalent to
	\[ 
	\begin{split}
	&2(1-\sigma_t)D^2\sigma_t + \frac{D^2}{2}\sigma_t^2+(\eta \| \Bg_{t+1} 
	\|/2)^2 \\
	\le& 2D^2 \sigma_{t+1} + (\eta \| \Bg_{t+1} \|/2)^2 - \eta \| \Bg_{t+1} \| 
	\sqrt{2D^2 \sigma_{t+1} }.
	\end{split}
	\]
Before taking the square root of both sides, we need the following 
\cref{lem:aux_inequality}.
\begin{lemma}\label{lem:aux_inequality}
	Under the assumptions of \cref{thm:regret_bound}, $ \sqrt{2D^2 
		\sigma_{t+1}} \geq \eta \| \Bg_{t+1} \|/2 $ holds for any $ 1\leq 
	t\leq 
	T $.
\end{lemma}
\begin{proof}
	By the definition of $ \Bg_{t+1} $, we have $ \| \Bg_{t+1} \| \leq 
	nM/\delta $. It suffices to show $ \sqrt{2D^2\sigma_{t+1}} \geq n\eta 
	M/(2\delta) $. By the definition of $ \sigma_{t+1} $, $ \eta $, and $ 
	\delta $, it is equivalent to $ 4T^{3/5}-(t+1)^{1/5}\geq 0 $. Since $ 1\leq 
	t 
	\leq T $, we only need to show $ 4T^{3/5}-(T+1)^{1/5}\geq 0 $. We define $ 
	f(T)=4T^{3/5}-(T+1)^{1/5} $. Its derivative is $f'(T)= \frac{12 
		(T+1)^{4/5}-T^{2/5}}{5 T^{2/5} (T+1)^{4/5}} $. We have \[ 
	\frac{12(T+1)^{4/5}}{T^{2/5}} = 12\left(T+\frac{1}{T}+2\right)^{2/5}\geq 
	12\cdot 4^{2/5}\geq 1
	\]
	if $ T\geq 1 $. Therefore, we know that $ f'(T)\geq 0 $ if $ T\geq 1 $. 
	Thus $ 
	f $ is non-decreasing on $ [1,\infty] $. This immediately yields $ f(T)\geq 
	f(1)\geq 0 $, which completes the proof.
\end{proof}
	Since $ \sqrt{2D^2 \sigma_{t+1}} \geq \eta \| \Bg_{t+1} \|/2 $, taking the 
	square root of both sides, we obtain $ 
	\sqrt{2(1-\sigma_t)D^2\sigma_t + \frac{D^2}{2}\sigma_t^2+(\eta \| \Bg_{t+1} 
		\|/2)^2} \leq \sqrt{2D^2 \sigma_{t+1}} - \eta \| \Bg_{t+1} \|/2,
	$
	which is equivalent to 
	\begin{equation}
	\label{eq:bound_h_t+1}
	\begin{split}
	&\sqrt{2(1-\sigma_t)D^2\sigma_t + \frac{D^2}{2}\sigma_t^2+(\eta \| 
	\Bg_{t+1} 
		\|/2)^2} + \eta \| \Bg_{t+1} \|/2 \\
	\leq& \sqrt{2D^2 \sigma_{t+1}}.
	\end{split}
	\end{equation}
	
	We define $ h_t(\Bx) \triangleq F_t(\Bx)-F_t(\Bx_t^*) $ and $ h_t 
	\triangleq 
	h_t(\Bx_t) $. We have \begin{align*}
	&h_t(\Bx_{t+1})  \\
	=& F_t(\Bx_{t+1}) - F_t(\Bx_t^*) \\
	=&  F_t((1-\sigma_t)\Bx_t+\sigma_t \Bv_t)-F_t(\Bx_t^*)\\
	= & F_t( \Bx_t+\sigma_t(\Bv_t-\Bx_t) )-F_t(\Bx_t^*)\\
	\leq & F_t(\Bx_t) - F_t(\Bx_t^*) + \sigma_t \nabla F_t(\Bx_t)^\top 
	(\Bv_t-\Bx_t) + D^2\sigma_t^2/2\\
	\le & F_t(\Bx_t) - F_t(\Bx_t^*) + \sigma_t \nabla F_t(\Bx_t)^\top 
	(\Bx_t^*-\Bx_t) + D^2\sigma_t^2/2\\
	\le & F_t(\Bx_t) - F_t(\Bx_t^*) + \sigma_t  
	(F_t(\Bx_t^*)-F_t(\Bx_t) ) + D^2\sigma_t^2/2\\
	= & (1-\sigma_t) (F_t(\Bx_t)-F_t(\Bx_t^*) ) + 
	D^2\sigma_t^2/2\\
	=&(1-\sigma_t)h_t+D^2\sigma_t^2/2.
	\end{align*}
	By the definition of $ h_t $ and $ F_t $ and in light of the fact that $ 
	\Bx_t^* $ is the minimizer of $ F_t $, we obtain \begin{align*}
	h_{t+1} = & F_t(\Bx_{t+1}) - F_t(\Bx_{t+1}^*) + \eta \Bg_{t+1}(\Bx_{t+1}- 
	\Bx_{t+1}^*)\\
	\le&  F_t(\Bx_{t+1}) - F_t(\Bx_{t}^*) + \eta \Bg_{t+1}(\Bx_{t+1}- 
	\Bx_{t+1}^*)\\
	= & h_t(\Bx_{t+1}) + \eta \Bg_{t+1}(\Bx_{t+1}- 
	\Bx_{t+1}^*) \\
	\le&  h_t(\Bx_{t+1}) + \eta \| \Bg_{t+1} \| \| \Bx_{t+1}- 
	\Bx_{t+1}^* \|.
	\end{align*}
	Notice that $ F_t $ is $ 2 $-strongly convex and that $ 
	\Bx_t^* $ is the minimizer of $ F_t $. We have
	$
	\| \Bx-\Bx_t^* \|^2 \leq F_t(\Bx)-F_t(\Bx_t^*)$.
	Therefore we obtain \begin{dmath*}
	h_{t+1}\leq   (1-\sigma_t)h_t + D^2\sigma_t^2/2 
	 + \eta \| 
	\Bg_{t+1} 
	\|\sqrt{F_{t+1}(\Bx_{t+1})-F_{t+1}(\Bx_{t+1}^*)}
	=  (1-\sigma_t)h_t + D^2\sigma_t^2/2 + \eta \| \Bg_{t+1} 
	\|\sqrt{h_{t+1}}.
	\end{dmath*}
	We will show $ h_{\tau}\leq 2D^2 \sigma_{\tau} $ holds for $ \forall 1\leq 
	\tau \leq T $ by 
	induction. Since $ h_1=F_1(\Bx_1)- F_1(\Bx_1^*)=0 $, it holds if $ t=1 $. 
	Assume that it holds for $ \tau =t  $. Now we set $ \tau = t+1 $. By the 
	induction hypothesis, we have \[
	h_{t+1}\leq 2(1-\sigma_t)D^2 \sigma_t + D^2\sigma_t^2/2 + \eta \| \Bg_{t+1} 
	\|\sqrt{h_{t+1}}.
	\]
	By completing the square, we obtain $
	(\sqrt{h_{t+1}} - \eta \| \Bg_{t+1}  
	\|/2)^2 \leq 2(1-\sigma_t)D^2 \sigma_t + D^2\sigma_t^2/2 + (\eta \| 
	\Bg_{t+1}  
	\|/2)^2.
	$
	Therefore,\[  
	\begin{split}
	\sqrt{h_{t+1}}\leq &\sqrt{2(1-\sigma_t)D^2 \sigma_t + 
		D^2\sigma_t^2/2 + (\eta \| 
		\Bg_{t+1}  
		\|/2)^2} \\
	& \qquad + \eta \| \Bg_{t+1}  
	\|/2.
	\end{split} \]
	By \cref{eq:bound_h_t+1}, the right-hand side is at most $ 
	\sqrt{2D^2\sigma_{t+1}} $. Thus we conclude that $ h_{t+1}\leq 2D^2 
	\sigma_{t+1} $. Then we are able to bound $ \| \Bx_t-\Bx_t^* \| $ as 
	follows: $ 
	\| \Bx_t-\Bx_t^* \| \leq \sqrt{F_t(\Bx_t)-F_t(\Bx_t^*)} \leq 
	\sqrt{2D^2\sigma_{t}} = \sqrt{2}Dt^{-1/5}.
	$
	By	 \cref{eq:bound_regret}, and since $ \| \nabla 
	\hat{f}_{t,\delta}(\Bx_t) \| \leq 
	\expect_{\Bv \sim B^n}[ \| \nabla f_t(\Bx_t + \delta \Bv) \|]
	\leq G $, we obtain \begin{align*}
	&\sum_{t=1}^{T} \expect[f_t(\Bx_t)-f_t(\Bz)] \\
	\leq &  D^2/\eta + 2\eta n^2M^2 
	T/\delta^2 + G\sum_{t=1}^{T}
	\expect[\|\Bx_t-  \Bx^*_t\|]+2\delta GT\\
	\le & \sqrt{2}nMDT^{4/5} + \frac{\sqrt{2}nMD}{c^2}T^{3/5} + 
	\frac{5\sqrt{2}}{4} 
	DGT^{4/5} \\
	& \qquad + 
	2cGT^{4/5}\\
	= & \frac{\sqrt{2}nMD}{c^2}T^{3/5} + (\sqrt{2}nMD+ 
	\frac{5\sqrt{2}}{4}DG+2cG)T^{4/5}.
	\end{align*}
	In the above equation, we use the fact that $ \sum_{t=1}^{T} t^{-1/5}\leq 
	\frac{5}{4}T^{4/5} $.
	Adding \cref{eq:regret1} to the inequality above, we have
	\begin{align}
 &	\sum_{t=1}^{T} \expect[f_t(\By_t)-f_t(\Bz)]\label{eq:regret_in_shrunk} \\
	 \leq &  \frac{\sqrt{2}nMD}{c^2}T^{3/5} + 
	(\sqrt{2}nMD+ \frac{5\sqrt{2}}{4}DG+3cG)T^{4/5}.
	\end{align}
Let $ \Bx^*\triangleq \argmin_{\Bx\in \constraint} \sum_{t=1}^{T} f_t(\Bx) $ 
and $ \Pi(\Bx^*) \triangleq \argmin_{\Bx\in 
(1-\alpha)\constraint} \| \Bx-\Bx^* \| $. We have 
$
\| \Bx^*-\Pi(\Bx^*) \| \le \| \Bx^* - (1-\alpha)\Bx^* \| \le \alpha R.$
If we set $ \Bz=\Pi(\Bx^*) $ in \cref{eq:regret_in_shrunk}, we have
\begin{align*}
&	\sum_{t=1}^{T} \expect[f_t(\By_t)-f_t(\Bx^*)] \\
= &	\sum_{t=1}^{T} 
\expect[f_t(\By_t)-f_t(\Pi(\Bx^*))+f_t(\Pi(\Bx^*))-f_t(\Bx^*)] \\
\leq &  \frac{\sqrt{2}nMD}{c^2}T^{3/5} + 
(\sqrt{2}nMD+ \frac{5\sqrt{2}}{4}DG+3cG)T^{4/5}\\
&+\alpha R G T.
\end{align*}
In light of $ \alpha=\delta/r $, we conclude that the regret is at most 
\[
\frac{\sqrt{2}nMD}{c^2}T^{3/5} + 
(\sqrt{2}nMD+ \frac{5\sqrt{2}}{4}DG+3cG+cRG/r)T^{4/5}.
\]

\section{Further Related Work}
\cite{zinkevich2003online} introduced the online convex optimization (OCO) 
problem and proposed online gradient descent.
 OCO generalizes existing models of online learning, 
including the universal portfolios model~\citep{cover1991universal} and 
prediction from expert advice~\citep{littlestone1994weighted}.
For strongly convex functions, an algorithm that achieves a logarithmic regret 
was proposed in
 \citep{hazan2007logarithmic}. Regularization-based methods applied to OCO 
 problems were investigated in ~\citep{grove2001general, 
kivinen1998relative}. The follow-the-perturbed-leader algorithm was introduced 
and analyzed in 
 \citep{kalai2005efficient}. Thereafter, the follow-the-regularized-leader 
 (FTRL) 
 was 
 independently considered in
%
~\citep{shalev2007online, shalev2007primal}
and 
 \citep{abernethy2008competing}. 
\citet{hazan2010extracting} showed the equivalence of FTRL 
and online mirror descent.

For projection-free convex optimization,
the Frank-Wolfe algorithm 
(also known as the conditional gradient method) was originally proposed in 
\citep{Frank1956algorithm}, and was further analyzed in 
\citep{jaggi2013revisiting}.  
The online conditional gradient method was investigated
 in \citep{hazan2012projection}.  A distributed online conditional 
 gradient 
 algorithm was proposed in \citep{Zhang2017Projection}. Conditional gradient 
 methods are very 
 sensitive to noisy gradients. This issue was recently resolved in centralized 
 \citep{mokhtari2018stochastic} and online settings 
 	\citep{Chen2018Projection}. 

A special case of bandit convex optimization (BCO) with linear objectives was 
studied in
\citep{Awerbuch2008Online,bubeck2012towards,karnin2014hard}. The general 
problem 
of 
BCO was considered in 
\citep{Flaxman2005Online} and was further studied in 
~\citep{dani2008price, 
	agarwal2011stochastic, bubeck2012regret,Bubeck2016Multi}.
 A near-optimal regret algorithm for the BCO 
problem with 
strongly-convex and smooth losses was introduced in 
\citep{Hazan2014Bandit}, while BCO with Lipschitz-continuous convex losses 
 was analyzed in \citep{kleinberg2005nearly}. Regret rate 
$\tilde{O}(T^{2/3})$ was achieved in \citep{saha2011improved} for convex and 
smooth loss functions, and in \citep{Agarwal2010Optimal} for strongly-convex 
loss functions, and was improved to $\tilde{O}(T^{5/8})$ in 
\citep{dekel2015bandit}. For strongly-convex and smooth loss functions, a 
lower 
bound of $\Omega(\sqrt{T})$ was attained in \citep{shamir2013complexity}.
\citet{Bubeck2017Kernel} proposed the first
  $ \mathrm{poly}(n)\sqrt{T} $-regret  algorithm whose running time is 
  polynomial in  
  horizon $ T $. 
   Zero-order optimization is relevant to BCO. 
  Interested 
  readers are referred to 
  \citep{conn2009introduction,duchi2015optimal,yu2016derivative}.
%

\section{Conclusion}
In this paper, we  presented the first computationally efficient 
projection-free bandit convex optimization 
algorithm  that requires no knowledge of the horizon $T$ and achieve an 
expected regret 
at most 
 $ O(nT^{4/5})$, where $ n $ is the dimension. Our experimental results show 
 that our 
proposed algorithm exhibits a sublinear regret and runs significantly faster 
than the other baselines. 
\bibliographystyle{plainnat}
{\fontsize{9.0pt}{10.0pt} \selectfont\bibliography{reference-list}}
\clearpage
\onecolumn
\begin{appendices}
	  \crefalias{section}{appsec}
	\section{Proof of $ \Bg_t^\top (\Bx_t^*-\Bx_{t+1}^*) \leq 2\eta 
	\|\Bg_t\|^{2} $
	}\label{app:Bregman}
	\begin{lemma}[Theorem 5.1 in~\citep{hazan2016introduction}]
		\label{lem:Bregman}
		Let $ \Bx_t^*=\argmin_{\Bx\in (1-\alpha)\constraint} F_t(\Bx) $. 
		We have
		$\Bg_t^\top (\Bx_t^*-\Bx_{t+1}^*) \leq 2\eta \|\Bg_t\|^{2}$.	
	\end{lemma}
	\begin{proof}
		We denote the regularizer in line~\ref{ln:F_t} of  
		\cref{alg:conditional_bandit} by $R(\Bx) \triangleq
		\|\Bx - \Bx_1\|^2$ and define the Bregman divergence with respect the 
		function $F$ by 
		\begin{equation}
		\label{equ:bregman}
		B_F(\Bx \| \By) = F(\Bx) - F(\By) - \nabla F(\By)^\top (\Bx-\By).
		\end{equation}
		Since $\Bx_{t+1}^*$ is a minimizer of $F_{t+1}$ and $F_{t+1}$ is 
		convex, we have	 
		\[
		\begin{split}	 
		F_{t+1}(\Bx_t^*) &= F_{t+1}(\Bx_{t+1}^*) + (\Bx_t^* - \Bx_{t+1}^*)^\top 
		\nabla F_{t+1}(\Bx_{t+1}^*) \\
		& \qquad + B_{F_{t+1}}(\Bx_t^* \| \Bx_{t+1}^*) \\
		&\geq F_{t+1}(\Bx_{t+1}^*) + B_{F_{t+1}}(\Bx_t^* \| \Bx_{t+1}^*) \\
		&= F_{t+1}(\Bx_{t+1}^*) + B_{R}(\Bx_t^* \| \Bx_{t+1}^*)
		\end{split}
		\]
		In the last equation, we use the fact that the Bregman divergence is 
		not 
		influenced by the linear terms in $ F $. Using again the fact that 
		$\Bx_t^*$ is the minimizer of $F_t$,
		we further deduce
		\[
		\begin{split}
		B_{R}(\Bx_t^* \| \Bx_{t+1}^*) & \leq F_{t+1}(\Bx_t^*) - 
		F_{t+1}(\Bx_{t+1}^*) \\
		& =  (F_t(\Bx_t^*) - F_t(\Bx_{t+1}^*)) + \eta \Bg_t^\top (\Bx_t^* - 
		\Bx_{t+1}^*) \\
		& \leq \eta \Bg_t^\top (\Bx_t^* - \Bx_{t+1}^*).
		\end{split}
		\]	 
		On the other hand, applying Taylor's theorem in several variables with 
		the 
		remainder given in Lagrange's form, we know that there exists $ 
		\bm{\xi}_t \in 
		[\Bx^*_t, \Bx^*_{t+1}] \triangleq \{ \lambda \Bx^*_t +(1-\lambda) 
		\Bx^*_{t+1}:\lambda\in [0,1] \} $ such that \[ 
		B_R(\Bx^*_t\| \Bx^*_{t+1} ) = \frac{1}{2} (\Bx^*_t-\Bx^*_{t+1})^\top 
		\BH(\bm{\xi}_t) (\Bx^*_t-\Bx^*_{t+1}),
		\]
		where $ \BH(\bm{\xi}_t) $ denotes the Hessian matrix of $ R $ at point 
		$ 
		\bm{\xi}_t $. Notice that the Hessian matrix of $ R $ is the identity 
		matrix everywhere. Therefore $ B_R(\Bx^*_t\| \Bx^*_{t+1} ) = 
		\frac{1}{2} 
		\| \Bx^*_t-\Bx^*_{t+1}\|^2 
		$.
		By Cauchy-Schwarz inequality, we obtain
		\[
		\begin{split}
		\Bg_t^\top (\Bx_t^*-\Bx_{t+1}^*) & \leq \|\Bg_t\| \cdot \|\Bx_t^* - 
		\Bx_{t+1}^*\|
		\\
		&= \|\Bg_t\| \cdot \sqrt{2B_R(\Bx_t^*\|\Bx_{t+1}^*)} \\
		&\leq \|\Bg_t\| \cdot \sqrt{2\eta \Bg_t^\top(\Bx_t^*-\Bx_{t+1}^*)}
		\end{split},
		\]
		which immediately yields
		$$\Bg_t^\top (\Bx_t^*-\Bx_{t+1}^*) \leq 2\eta \|\Bg_t\|^{2}$$
		
	\end{proof}

\section{Proof of \cref{thm:anytime}}
\label{app:anytime}
\begin{proof}
	The regret of \cref{alg:anytime} by the end of the $ t $-th iteration is at 
	most \[ 
	\begin{split}
	\sum_{m=0}^{ \lceil \log_2(t+1) \rceil -1} \beta (2^m)^{4/5} &= \beta 
	\frac{\left(2^{\lceil \log_2(t+1) \rceil}\right)^{4/5}-1}{2^{4/5}-1} \\
	&\leq 
	\frac{\beta}{1-2^{-4/5}} (t+1)^{4/5}.
	\end{split}
	\]
\end{proof}
\end{appendices}
\end{document}